\documentclass[10pt,twocolumn,letterpaper]{article}

\usepackage{cvpr}
\usepackage{times}
\usepackage{epsfig}
\usepackage{graphicx}
\usepackage{amsmath}
\usepackage{amsthm}
\usepackage{mathrsfs} 
\usepackage{amssymb}

\theoremstyle{plain}
\newtheorem{thm}{Theorem} % reset theorem numbering for each chapter

\theoremstyle{definition}
\newtheorem{defn}{Definition} % definition numbers are dependent on theorem numbers

\theoremstyle{corollary}
\newtheorem{coro}{Corollary} % definition numbers are dependent on theorem numbers

\usepackage[pagebackref=true,breaklinks=true,letterpaper=true,colorlinks,bookmarks=false]{hyperref}

\cvprfinalcopy % *** Uncomment this line for the final submission

\def\cvprPaperID{2467} % *** Enter the CVPR Paper ID here

\def\V{\mathcal{V}}

\def\Rad{\mathscr{R}}
\def\R{\mathbb{R}}
\def\S{\mathbb{S}}

\ifcvprfinal\fi
\begin{document}

%%%%%%%%% TITLE
\title{Sliced Wasserstein Kernels for Probability Distributions}

\author{Soheil Kolouri\\
Carnegie Mellon University\\
{\tt\small skolouri@andrew.cmu.edu}
\and
Yang Zou\\
Carnegie Mellon University\\
{\tt\small yzou2@andrew.cmu.edu}
\and
Gustavo K. Rohde\\
Carnegie Mellon University\\
{\tt\small gustavor@cmu.edu}
}

\maketitle
%\thispagestyle{empty}

%%%%%%%%% ABSTRACT
\begin{abstract}

Optimal transport distances, otherwise known as Wasserstein distances, have recently drawn ample attention in computer vision and machine learning as a powerful discrepancy measure for probability distributions. The recent developments on alternative formulations of the optimal transport have allowed for faster solutions to the problem and has revamped its practical applications in machine learning. In this paper, we exploit the widely used kernel methods and provide a family of provably positive definite kernels based on the Sliced Wasserstein distance and demonstrate the benefits of these kernels in a variety of learning tasks. Our work provides a new perspective on the application of optimal transport flavored distances through kernel methods in machine learning tasks.

\end{abstract}

%%%%%%%%% BODY TEXT
\section{Introduction}
Many computer vision algorithms rely on characterizing images or image features as probability distributions which are often high-dimensional. This is for instance the case for histogram-based methods like the Bag-of-Words (BoW) \cite{jegou2010improving}, feature matching \cite{hauagge2012image}, co-occurence matrices in texture analysis \cite{gomez2012analysis}, action recognition \cite{wang2013dense}, and many more. Having an adequate measure of similarity (or equivalently discrepancy) between distributions becomes crucial in these applications.  The classic distances or divergences for probability densities include Kullback-Leibler divergence,  Kolmogorov distance, Bhattacharyya distance (also known as the Hellinger distance), etc. More recently, however, the optimal transportation framework and the Wasserstein distance \cite{villani2008optimal} also known as the Earth Mover Distance (EMD) \cite{rubner2000earth} have attracted ample interest in the computer vision \cite{solomon2015convolutional}, machine learning \cite{cuturi2013sinkhorn}, and biomedical image analysis \cite{basu2014detecting} communities. The Wasserstein distance computes the optimal warping to map a given input probability measure $\mu$ to a second one $\nu$. The optimality corresponds to a cost function which measures the expected value of the displacement in a  warping. Informally, thinking about $\mu$ and $\nu$ as piles of dirt (or sand), the Wasserstein distance measures a notion of displacement of each sand particle in $\mu$ times its mass to warp $\mu$ into $\nu$. 

The Wasserstein distance has been shown to provide a useful tool to quantify the geometric discrepancy between two distributions. Specifically, they've been used as distances in content-based image image retrieval \cite{rubner2000earth}, modeling and visualization of modes of variation of image intensity values \cite{basu2014detecting,wang2013linear,seguyprincipal,bigot2015geodesic},  estimate the mean of a family of probability measures (i.e. barycenters of distributions) \cite{agueh2011barycenters,rabin2012wasserstein}, modeling and visualization of modes of variation of image intensity values \cite{basu2014detecting,wang2013linear}, cancer detection \cite{ozolek2014accurate,tosun2015detection}, super-resolution \cite{kolouri2015transport}, amongst other applications. Recent advances utilizing variational minimization \cite{haker2004optimal, chartrand2009gradient}, particle approximation \cite{wang2011optimal}, and entropy regularization \cite{cuturi2013sinkhorn,solomon2015convolutional}, have enabled  transport metrics to be efficiently applied to machine learning and computer vision problems. In addition, Wang et al. \cite{wang2013linear} described a method for computing a transport distance (denoted as linear optimal transport) between all image pairs of a dataset of $N$ images that requires only $N$ transport minimization problems. Rabin et al. \cite{rabin2012wasserstein} and  Bonneel et al. \cite{bonneel2015sliced}  proposed to leverage the fact that these problems are easy to solve for one-dimensional distributions, and introduced an alternative distance called the Sliced Wasserstein distance. Finally, recent work  \cite{park2015cumulative,kolouri2015radon} has shown that the transport framework can be used as an invertible signal/image transformation framework that can render signal/image classes more linearly separable, thus facilitating a variety of pattern recognition tasks.

Due to the benefits of using the transport distances outlined above, and given the flexibility and power of kernel-based methods \cite{hofmann2008kernel}, several methods using transport-related distances in constructing kernel matrices have been described with applications in computer vision, and EEG data classification \cite{zhang2007local,daliri2013kernel}. Since positive definite RBF kernels require the metric space induced by the distance to be `flat' (zero curvature) \cite{Feragen_2015_CVPR}, and because the majority of the transport-related distances mentioned above, in particular distances for distributions of dimension higher than one utilizing the $L_2$ cost, do not satisfy this requirement, few options for provably positive definite transport-based kernels have emerged. Cuturi \cite{cuturi2007permanents} for example, suggested utilizing the permanent of the transport polytope, thus guaranteeing the positivity of the derived kernels. More recently, Gardner et al \cite{gardner2015earth} have shown that certain types of earth mover's distances (e.g. the 0-1 distance) can yield kernels which are positive definite.

Here we expand upon these sets of ideas and show that the Sliced Wasserstein distance satisfies the basic requirements for being used as positive definite kernels \cite{hofmann2008kernel} in a variety of regression-based pattern recognition tasks, and have concrete theoretical and practical advantages. Based on recent works on kernel methods \cite{jayasumana2013kernel, jayasumana2015kernel,Feragen_2015_CVPR}, we exploit the connection between the optimal transport framework and the kernel methods and introduce a family of provably positive definite kernels which we denote as Sliced Wasserstein kernels. We derive mathematical results enabling the application of the Sliced Wasserstein metric in the kernel framework and, in contrast to other work, we describe the explicit form for the embedding of the kernel, which is analytically invertible. Finally, we demonstrate experimentally advantages of the Sliced Wasserstein kernels over the commonly used kernels such as the radial basis function (RBF) and the polynomial kernels for a variety of regression.

%We first show that the Gaussian Sliced Wasserstein kernel is indeed a positive definite kernel. Then, we demonstrate that for the Sliced Wasserstein distance the map to the kernel-space (also called the feature-space) is defined explicitly and hence a wide variety of kernels including polynomial kernels can be defined based on this mapping to the feature-space. 

\textbf{Paper organization.} We first review the preliminaries and formally present the $L^p$-Wasserstein distance, the Sliced Wasserstein distance, and review some of the theorems in the literature on positive definiteness of kernels in Section \ref{sec:background}. The main theorems of the paper on the Sliced Wasserstein kernels are stated in Section \ref{sec:swkernels}. In Section \ref{sec:methods} we review some of the kernel-based algorithms including the kernel $k$-means clustering, the kernel PCA, and the kernel SVM.   Section \ref{sec:results} demonstrates the benefits of the Sliced Wasserstein kernel  over the commonly used kernels in a variety of pattern recognition tasks. Finally we conclude our work in Section \ref{sec:conclusion} and lay out future directions for research in the area. 

\section{Background}
\label{sec:background}
\subsection{The $L^p$-Wasserstein distance}
Let $\sigma$ and $\mu$ be two probability measures on measurable spaces $X$ and $Y$ and their corresponding probability density
functions $I_0$ and $I_1$. 

\begin{defn}
 The $L^p$-Wasserstein distance for $p\in [1,\infty)$ is defined as,
\begin{eqnarray}
W_p(\sigma,\mu):=(\operatorname{inf}_{\pi\in\Pi(\sigma,\mu)} \int_{X\times Y} (x-y)^p d\pi(x,y))^{\frac{1}{p}}
\end{eqnarray}
where $\Pi(\sigma,\mu)$ is the set of all transportation plans, $\pi\in \Pi(\sigma,\mu)$, that satisfy the following,
\begin{eqnarray}
\begin{array}{lr}
\pi(A \times Y)= \sigma(A) & \forall A\subseteq X\\
\pi(X \times B)= \mu(B) & \forall B\subseteq Y
\end{array}
\end{eqnarray}
Due to Brenier's theorem \cite{brenier1991polar}, for  absolutely continuous probability measures $\sigma$ and $\mu$   (with respect to Lebesgue measure) the $L^p$-Wasserstein distance can be  equivalently obtained from, 
\begin{eqnarray}
W_p(\sigma,\mu)=(\operatorname{inf}_{f\in MP(\sigma,\mu)} \int_{X} (f(x)-x)^p d\sigma(x))^{\frac{1}{p}}
\end{eqnarray}
where $MP(\sigma,\mu)=\{ f:X\rightarrow Y ~|~ f_{\#}\sigma=\mu\}$. 
\end{defn}
 In the one-dimensional case, the $L^2$-Wasserstein distance has a closed form solution as the MP transport map, $f_{\#}\sigma=\mu$, is unique. We will show this in the following theorem.
\begin{thm}
\label{thm:unique}
Let $\sigma$ and $\mu$ be absolutely continuous probability measures on $\R$ with corresponding positive density functions $I_0$ and $I_1$, and corresponding cumulative distribution functions $CDF_\sigma(x):=\sigma((-\infty,x))$ and $CDF_\mu(x):=\mu((-\infty,x))$.  Then, there only exists one monotonically increasing transport map $f:\R\rightarrow\R$ such that $f_{\#}\sigma=\mu$ and it is defined as, 
\begin{eqnarray}
f(x):= min\{ t\in \R:~ CDF_\mu(t)\geq CDF_\sigma(x)\}
\end{eqnarray}
and the $L^2$-Wasserstein distance is obtained from,  
\begin{eqnarray}
W_2(\sigma,\mu)&=&( \int_{X} (f(x)-x)^2 d\sigma(x))^{\frac{1}{2}}\nonumber\\
&=&( \int_{X} (f(x)-x)^2 I_0(x)dx)^{\frac{1}{2}}.
\end{eqnarray}
Note that throughout the paper we use $W_2(\sigma,\mu)$ and $W_2(I_0,I_1)$ interchangeably. 
\begin{proof}
Assume there exist more than one transport maps, say $f$ and $g$, such that $f_\#\sigma=g_\#\sigma=\mu$, then we can write,
\begin{eqnarray*}
\int_{-\infty}^{f(x)} I_1(\tau)d\tau=\int_{-\infty}^x I_0(\tau)d\tau=\int_{-\infty}^{g(x)} I_1(\tau)d\tau
\end{eqnarray*}
Above is equivalent to $CDF_\mu(f(x))=CDF_\mu(g(x))$, but $I_1$ is positive everywhere, hence the CDF is monotonically increasing, therefore $CDF_\mu(f(x))=CDF_\mu(g(x))$ implies that $f(x)=g(x)$. Therefore the transport map in one dimension is unique.
\end{proof}
\end{thm}
The closed-form solution of the Wasserstein distance in one dimension is an attractive property, as it alleviates the need for often computationally intensive optimizations. Recently there has been some work on utilizing this property of the Wasserstein distance to higher dimensional problems \cite{rabin2012wasserstein, bonneel2015sliced, kolouri2015radon} (i.e. images). We review such distances in the following section.

\subsection{The Sliced Wasserstein distance}

 The idea behind the Sliced Wasserstein distance is to first obtain a family of one-dimensional representations for a higher-dimensional probability distribution through projections, and then calculate the distance between two input distributions as a functional on the Wasserstein distance of their one-dimensional representations. In this sense, the distance is obtained by solving several one-dimensional optimal transport problems, which have closed-form solutions. 

\begin{defn}
The $d$-dimensional Radon transform $\Rad$ maps a function  $I\in L^1(\R^d)$ where $L^1(\R^d):=\{ I:\R^d \rightarrow \R | \int_{\R^d} |I(x)|dx \leq \infty\}$ into the set of its integrals over the hyperplanes of $\R^n$ and is defined as, 

\begin{eqnarray}
\Rad I(t,\theta):=\int_{\R^{d-1}} I(t\theta+\gamma\theta^{\perp})d\gamma
\end{eqnarray}
here $\theta^{\perp}$ is the subspace or unit vector orthogonal to $\theta$. Note that $\Rad: L^1(\R^d)\rightarrow L^1(\R\times \S^{d-1})$, where $\S^{d-1}$ is the unit sphere in $\R^{d}$. 
\end{defn}

We note that the Radon transform is an invertible, linear transform and we denote its inverse as $\Rad^{-1}$. For brevity we do not define the inverse Radon transform here, but the details can be found in \cite{natterer1986mathematics}. Next, following  \cite{rabin2012wasserstein, bonneel2015sliced, kolouri2015radon} we define the Sliced Wasserstein distance. 

\begin{defn}
 Let $\mu$ and $\sigma$  be two continuous probability measures on $\R^d$ with corresponding positive probability density functions $I_1$ and $I_0$.  The Sliced Wasserstein distance between $\mu$ and $\sigma$ is defined as, 
 \begin{eqnarray}
 SW(\mu,\nu)&:=&(\int_{\S^{d-1}} W^2_2(\Rad I_1(.,\theta),\Rad I_0(.,\theta)) d\theta)^{\frac{1}{2}} \nonumber\\
 &=& ( \int_{\S^{d-1}} \int_\R (f_\theta(t)-t)^2 \Rad I_0(t,\theta)  dt d\theta)^{\frac{1}{2}}\nonumber\\
 \end{eqnarray}
 where $f_\theta$ is the MP map between $\Rad I_0(.,\theta)$ and $\Rad I_1(.,\theta)$ such that,
 \begin{eqnarray}
\int_{-\infty}^{f_\theta(t)} \Rad I_1(\tau,\theta)  d\tau =\int_{-\infty}^t \Rad I_0(\tau,\theta)  d\tau,~~\forall \theta\in\S^{d-1}
 \end{eqnarray} 
 or equivalently in the differential form, 
\begin{eqnarray}
\frac{ \partial f_\theta(t)}{\partial t} \Rad I_1(f_\theta(t),\theta)=\Rad I_0(t,\theta),~~\forall \theta\in\S^{d-1}.
\end{eqnarray}
\end{defn}
The Sliced Wasserstein distance as defined above is symmetric, and it satisfies subadditivity and coincidence axioms, and hence it is a true metric (See \cite{kolouri2015radon} for proof).
 
\subsection{The Gaussian RBF Kernel on Metric Spaces}
The kernel methods and specifically the Gaussian RBF kernel has shown to be a very powerful tool in a plethora of applications. The Gaussian RBF kernel was initially designed for Euclidean spaces, however, recently there has been several work extending the Gaussian RBF kernel to other metric spaces. Jayasumana et al. \cite{jayasumana2015kernel}, for instance, developed an approach to exploit the Gaussian RBF kernel method on Riemannian manifolds. In another interesting work, Feragen et al. \cite{Feragen_2015_CVPR} considered the Gaussian RBF kernel on general geodesic metric spaces and showed that the geodesic Gaussian kernel is only positive definite when the underlying Riemannian manifold is flat or in other words it is isometric to a Euclidean space.  Here we review some definitions and theorems that will be used in the rest of the paper. 

\begin{defn}
A positive definite (PD) (resp. conditional negative definite) kernel on a set $M$ is a symmetric function $k:M\times M\rightarrow \R$, $k(I_i,I_j)=k(I_j,I_i)$ for all $I_i,I_j\in M$, such that for any $n\in N$, any elements $I_1,...,I_n \in X$, and any number $c_1,...,c_n\in \R$, we have
\begin{eqnarray}
\sum_{i=1}^n\sum_{j=1}^n c_ic_j k(I_i,I_j) \geq 0 ~~~\text{(resp. $\leq 0$)}
\end{eqnarray}
with the additional constraint of  $\sum_{i=1}^n c_i=0$ for the conditionally negative definiteness. 
\end{defn}

Above definition is used in the following important theorem due to Isaac J. Schoenberg \cite{schoenberg1938metric},
\begin{thm}
\label{thm:cnd}
Let $M$ be a nonempty set and $f:(M\times M)\rightarrow \R$ be a function. Then kernel $k(I_i,I_j)=exp(-\gamma f(I_i,I_j))$ for all $I_i,I_j\in M$ is positive definite for all $\gamma>0$ if and only if $f(.,.)$ is conditionally negative definite.
\end{thm}
The detailed proof of above theorem can be found in Chapter 3, Theorem 2.2 of \cite{berg1984harmonic}.

Following the work by Jayasumana et al. \cite{jayasumana2015kernel}, here we state the theorem (Theorem 6.1 in \cite{jayasumana2015kernel}) which provides the necessary and sufficient conditions for obtaining a positive definite Gaussian kernel from a given distance function defined on a generic metric space.
\begin{thm}
\label{thm:gausskernel}
Let $(M,d)$ be a metric space and define $k:M\times M\rightarrow \R$  by $k(I_i,I_j):=exp(-\gamma d^2(I_i,I_j))$ for all $I_i,I_j\in M$. Then $k(.,.)$ is a positive definite kernel for all $\gamma>0$  if and only if there exists an inner product space $\V$ and a function $\psi:M\rightarrow \V$ such that $d(I_i,I_j)=\|\psi(I_i)-\psi(I_j)\|_\V$.
\end{thm}
\begin{proof}
The detailed proof is presented in \cite{jayasumana2015kernel}. The gist of the proof, however, follows from Theorem \ref{thm:cnd} which states that positive definiteness of $k(.,.)$ for all $\gamma>0$ and conditionally negative definiteness of $d^2(.,.)$ are equivalent conditions. Hence, for $d(I_i,I_j)=\|\psi(I_i)-\psi(I_j)\|_\V$ it is straightforward to show that $d^2(.,.)$ is conditionally negative definite and therefore $k(.,.)$ is positive definite. On the other hand, if $k(.,.)$ is positive definite then $d^2(.,.)$ is conditionally negative definite, and since $d(I_i,I_i)=0$ for all $I_i\in M$ a vector space $\V$ exists for which $d(I_i,I_j)=\|\psi(I_i)-\psi(I_j)\|_\V$ for some $\psi:M\rightarrow \V$ \cite{jayasumana2015kernel,berg1984harmonic}.
\end{proof}

\section{Sliced Wasserstein Kernels}
\label{sec:swkernels}

In this section we present our main theorems. We first demonstrate that the Sliced Wasserstein Gaussian kernel of probability measures is a positive definite kernel. We proceed our argument by showing that there is an explicit  formulation for the nonlinear mapping to the kernel space (aka feature space) and define a family of kernels based on this mapping. 

We start by proving that for one-dimensional probability density functions the $L^2$-Wasserstein Gaussian kernel is a positive definite kernel. 

\begin{thm}
Let $M$ be the set of absolutely continuous one-dimensional positive probability density functions and define $k: M\times M\rightarrow \R$  to be $k(I_i,I_j):=exp(-\gamma W_2^2(I_i,I_j))$ then $k(.,.)$ is a positive definite kernel for all $\gamma>0$. 
\label{thm:wassersteinkernel}
\end{thm}
\begin{proof}
In order to be able to show this, we first show that for absolutely continuous one-dimensional positive probability density functions there exists an inner product space $\V$ and a function $\psi:M\rightarrow\V$ such that $W_2(I_i,I_j)=\|\psi(I_i)-\psi(I_j)\|_\V$.  

Let $\mu$, $\nu$, and $\sigma$ be probability measures on $\R$ with corresponding  absolutely continuous positive density functions $I_0$, $I_1$, and $I_2$.  Let $f,g,h:\R\rightarrow \R$ be transport maps such that $f_{\#}\sigma=\mu$, $g_{\#}\sigma=\nu$, and $h_{\#}\mu=\nu$. In the differential form this is equivalent to $f'I_1(f)=g'I_2(g)=I_0$ and $h'I_2(h)=I_1$ where $I_1(f)$ represents $I_1\circ f$. Then we have, 
\begin{eqnarray*}
W^2_2(I_1,I_0)&=&\int_\R (f(x)-x)^2I_0(x)dx,\\
W^2_2(I_2,I_0)&=&\int_\R (g(x)-x)^2I_0(x)dx,\\
W^2_2(I_2,I_1)&=&\int_\R (h(x)-x)^2I_1(x)dx.
\end{eqnarray*}
We follow the work of Wang et al \cite{wang2013linear} and Park et al. \cite{park2015cumulative} and define a nonlinear map with respect to a fixed probability measure, $\sigma$ with corresponding density $I_0$, that maps an input probability density to a linear functional on the corresponding transport map. More precisely, $\psi_{\sigma}(I_1(.)) := (f(.)-id(.))\sqrt{I_0(.)}$ where $id(.)$ is the identity map and $f'I_1(f)=I_0$. Notice that such $\psi_\sigma$ maps the fixed probability density $I_0$ to zero, $\psi_\sigma(I_0(.))= (id(.)-id(.))\sqrt{I_0(.)}=0$ and it satisfies,
\begin{eqnarray*}
W_2(I_1,I_0)&=&\|\psi_\sigma(I_1)\|_2\\
W_2(I_2,I_0)&=&\|\psi_\sigma(I_2)\|_2.
\end{eqnarray*}
More importantly, we demonstrate that $W_2(I_2,I_1)=\|\psi_\sigma(I_1)-\psi_\sigma(I_2)\|_2$. To show this, we can write,
 \begin{eqnarray*}
W^2_2(I_2,I_1)&=&\int_\R (h(x)-x)^2I_1(x)dx\\
&=& \int_\R (h(f(\tau))-f(\tau))^2f'(\tau)I_1(f(\tau))d\tau \\
&=& \int_\R (g(\tau)-f(\tau))^2I_0(\tau)d\tau \\
&=& \int_\R ((g(\tau)-\tau)-(f(\tau)-\tau))^2I_0(\tau)d\tau \\
&=& \|\psi_\sigma(I_1)-\psi_\sigma(I_2)\|^2_2
\end{eqnarray*}
where in the second line we used the change of variable $f(\tau)=x$. In the third line, we used the fact that composition of transport maps is also a transport map, in other words, since $f_{\#}\sigma=\mu$ and $h_{\#}\mu=\nu$ then $(h\circ f)_{\#}\sigma=\nu$. Finally, from Theorem \ref{thm:unique} we have that the one-dimensional transport maps  are unique, therefore if $(h\circ f)_{\#}\sigma=\nu$ and $g_{\#}\sigma=\nu$ then $h\circ f=g$.

We  showed that there exists a nonlinear map $\psi_\sigma:M\rightarrow \V$ for which $W_2(I_i,I_j)=\|\psi_\sigma(I_i)-\psi_\sigma(I_j)\|_2$ and therefore according to Theorem \ref{thm:gausskernel}, $k(I_i,I_j):= exp(-\gamma W^2_2(I_i,I_j))$ is a positive definite kernel. 
\end{proof}
Combining the results in Theorems \ref{thm:wassersteinkernel} and \ref{thm:cnd} will lead to the following corollary.
\begin{coro}
The  squared L$^2$-Wasserstein distance for  continuous one-dimensional positive probability density functions, $W^2_2(.,.)$, is a conditionally negative definite function. 
\label{coro:cnd}
\end{coro}

Moreover, following the work of Feragen et al \cite{Feragen_2015_CVPR}, Theorem \ref{thm:wassersteinkernel} also states that the Wasserstein space in one dimension (the space of one dimensional absolutely continuous positive probability densities endowed with the $L^2$-Wasserstein metric) is a flat space, in the sense that it is isometric to the Euclidean space. 

\subsection{The Sliced Wasserstein Gaussian kernel}

Now we are ready to show that the Sliced Wasserstein Gaussian kernel is a positive definite kernel. 
\begin{thm}
Let $M$ be the set of absolutely continuous positive probability density functions and define $k: M\times M\rightarrow \R$  to be $k(I_i,I_j):=exp(-\gamma SW^2(I_i,I_j))$ then $k(.,.)$ is a positive definite kernel for all $\gamma>0$. 
\label{thm:main}
\end{thm}
\begin{proof}
First note that for an absolutely continuous positive probability density function, $I\in M$, each hyperplane integral, $\Rad I(.,\theta)$, $\forall \theta\in\S^{d-1}$ is a one dimensional absolutely continuous positive probability density function. Therefore, following Corollary \ref{coro:cnd} for $I_1,...,I_N\in M$ we have, 
\begin{eqnarray}
\sum_{i=1}^N\sum_{j=1}^N c_ic_j W^2_2(\Rad I_i(.,\theta),\Rad I_j(.,\theta) )\leq 0, ~\forall\theta\in \S^{d-1}
\end{eqnarray}
where $\sum_{i=1}^N c_i=0$. Integrating the left hand side of above inequality over $\theta$ leads to, 
\begin{eqnarray}
\int_{\S^{d-1}}(\sum_{i=1}^N\sum_{j=1}^N c_ic_j W^2_2(\Rad I_i(.,\theta),\Rad I_j(.,\theta) )d\theta)\leq 0 &\Rightarrow& \nonumber\\
\sum_{i=1}^N\sum_{j=1}^N c_ic_j (\int_{\S^{d-1}} W^2_2(\Rad I_i(.,\theta),\Rad I_j(.,\theta) )d\theta)\leq 0 &\Rightarrow& \nonumber\\
\sum_{i=1}^N\sum_{j=1}^N c_ic_j SW^2(I_i,I_j)\leq 0
\end{eqnarray}
Therefore $SW^2(.,.)$ is conditionally negative definite, and hence from Theorem  \ref{thm:cnd} we have that $k(I_i,I_j):=exp(-\gamma SW^2(I_i,I_j))$ is a positive definite kernel for $\gamma>0$.
\end{proof}

The following corollary follows from Theorems \ref{thm:gausskernel} and \ref{thm:main}. 
\begin{coro}
Let $M$ be the set of absolutely continuous positive probability density functions and let $SW(.,.)$ be the sliced Wasserstein distance, then there exists an inner product space $\V$ and a function $\phi:M\rightarrow\V$ such that $SW(I_i,I_j)=\|\phi(I_i)-\phi(I_j)\|_\V$, for $\forall I_i, I_j \in M$.  
\label{coro:SW}
\end{coro}

In fact, using a similar argument as the one we provided in the proof of Theorem \ref{thm:wassersteinkernel} it can be seen that for a fixed absolutely continuous measure, $\sigma$, with positive probability density function $I_0$, we can define,  
\begin{eqnarray}
\phi_\sigma(I_i):= (f_i(t,\theta)- t)\sqrt{\Rad I_0(t,\theta)}
\label{eq:phi}
\end{eqnarray}
where $f_i$ satisfies $\frac{\partial f_i(t,\theta)}{\partial t}\Rad I_i(f_i(t,\theta),\theta)=\Rad I_0(t,\theta)$. It is straightforward to show that such $\phi_\sigma$ also satisfies the following,
\begin{eqnarray}
SW(I_i,I_0)&=&\| \phi_\sigma(I_i)\|_2\\
SW(I_i,I_j)&=&\| \phi_\sigma(I_i)-\phi_\sigma(I_j)\|_2
\end{eqnarray}
for a detailed derivation and proof of the equations above please refer to \cite{kolouri2015radon}. More importantly, such nonlinear transformation $\phi_\sigma:M\rightarrow \V$ is invertible. 

%This implies that instead of using the kernel trick, one can truly map the probability distribution into the vector space $\V$ and perform any statistical analysis in $\V$ and finally invert the findings back to the original space $M$ via $\phi^{-1}_\sigma}.

\subsection{The Sliced Wasserstein polynomial kernel}
In this section, using Corollary \ref{coro:SW} we define a polynomial Kernel based on the Sliced Wasserstein distance and show that this kernel is positive definite.  
\begin{thm}
Let $M$ be the set of absolutely continuous positive probability density functions and let $\sigma$ be a template probability measure with corresponding probability density function $I_0\in M$. Let $\phi_\sigma: M\rightarrow \V$ be defined as in Equation \eqref{eq:phi}. Define a kernel function  $k: M\times M\rightarrow \R$  to be $k(I_i,I_j):=(\langle \phi_\sigma(I_i),\phi_\sigma(I_j)\rangle)^d$ for $d\in\{1,2,...\}$ then $k(.,.)$ is a positive  definite kernel. 
\end{thm}
\begin{proof}
Given $I_1,...,I_N\in M$ and for $d=1$ we have,

\begin{eqnarray}
&& \sum_{i=1}^N\sum_{j=1}^N c_ic_j \langle \phi_\sigma(I_i),\phi_\sigma(I_j)\rangle=\nonumber\\&& \langle \sum_{i=1}^Nc_i\phi_\sigma(I_i),\sum_{j=1}^Nc_j\phi_\sigma(I_j)\rangle
= \|\sum_{i=1}^Nc_i \phi_\sigma(I_i)\|^2_2\geq 0\nonumber\\
\end{eqnarray}
and since $k(I_i,I_j)=\langle \phi_\sigma(I_i),\phi_\sigma(I_j)\rangle$ is a positive definite kernel and from Mercer's kernel properties it follows that $k(I_i,I_j)=(\langle \phi_\sigma(I_i),\phi_\sigma(I_j)\rangle)^d$ and $k(I_i,I_j)=(\langle \phi_\sigma(I_i),\phi_\sigma(I_j)\rangle+1)^d$ are also positive definite kernels. 
\end{proof}

\section{The Sliced Wasserstein Kernel-based  algorithms}
\label{sec:methods}

\subsection{Sliced Wasserstein Kernel $k$-means }

Considering clustering problems for data with the form of probability distributions, we propose the Sliced Wasserstein $k$-means. For a set of input data $I_1,...,I_N\in M$, the Sliced Wasserstein $k$-means with kernel $k(I_i,I_j)=\langle \phi_\sigma(I_i),\phi_\sigma(I_j) \rangle$ transforms the input data to the kernel space via $\phi_\sigma:  M\rightarrow \V$ and perform $k$-means in this space. Note that since $\|\phi_\sigma(I_i)-\phi_\sigma(I_j)\|_2=SW(I_i,I_j)$, performing $k$-means in $\V$ is equivalent to performing $k$-means with Sliced Wasserstein distance in $M$. The kernel $k$-means with the Sliced Wasserstein distance, essentially provides $k$ barycenters for the input distributions. In addition,  the Gaussian and polynomial Sliced Wasserstein kernel $k$-means are obtained by performing Gaussian and polynomial kernel $k$-means in $\V$.

\subsection{Sliced Wasserstein Kernel PCA }

Now we consider the key concepts of the kernel PCA. The Kernel-PCA \cite{scholkopf1997kernel} is a non-linear dimensionality reduction method that can be interpreted as applying the PCA in the kernel-space (or feature-space), $\V$.  Performing standard PCA on $\phi_\sigma(I_1),...,\phi_\sigma(I_N)\in\V$ provides the Sliced Wasserstein kernel PCA.  In addition, applying Gaussian and polynomial Sliced Wasserstein kernel PCA to $I_1,...,I_N\in M$ is also equivalent to applying  Gaussian and polynomial kernel PCA on $\phi_\sigma(I_1),...,\phi_\sigma(I_N)\in\V$.  We utilize the cumulative percent variance (CPV) as a quality measure for how well the principal components are capturing the variation of the dataset in $M$ and similarly in $\V$. 

\subsection{Sliced Wasserstein Kernel SVM}

For a binary classifier, given a set of training examples $\{I_i,y_i\}_{i=1}^N$ where $I_i\in M$ and $y_i\in\{-1,1\}$, support vector machine (SVM) searches for a hyperplane in $M$ that separates training classes while maximizing the separation margin where separation is measured with the Euclidean distance. A kernel-SVM , on the other hand, searches for a hyperplane in $\V$ which provides maximum margin separation between $\phi_\sigma(I_i)$s which is equivalent to finding a nonlinear classifier in $M$ that maximizes the separation margin according to the Sliced Wasserstein distance. Note that the kernel SVM with the Sliced Wasserstein Gaussian and polynomial kernels are essentially equivalent to applying kernel SVM, with the same kernels in the transformed Sliced Wasserstein space $\V$. It is worthwhile to mention that, since $\phi_\sigma$ is invertible, the Sliced Wasserstein kernel SVM learned from kernel $k(I_i,I_j)=\langle \phi_\sigma(I_i),\phi_\sigma(I_j)\rangle$ can be sampled along side the orthogonal direction to the discriminating hyperplane in $\V$ and then inverted through $\phi^{-1}_\sigma$ to directly get the discriminating features in the space of the probability densities, $M$. Finally, for multiclass classification problems, the problem can be turned into several binary classification tasks using pairwise coupling as suggested by Wu et al. \cite{wu2004probability}  or a one versus all approach. 

\begin{figure*}[t]
\centering
\includegraphics[width=\linewidth]{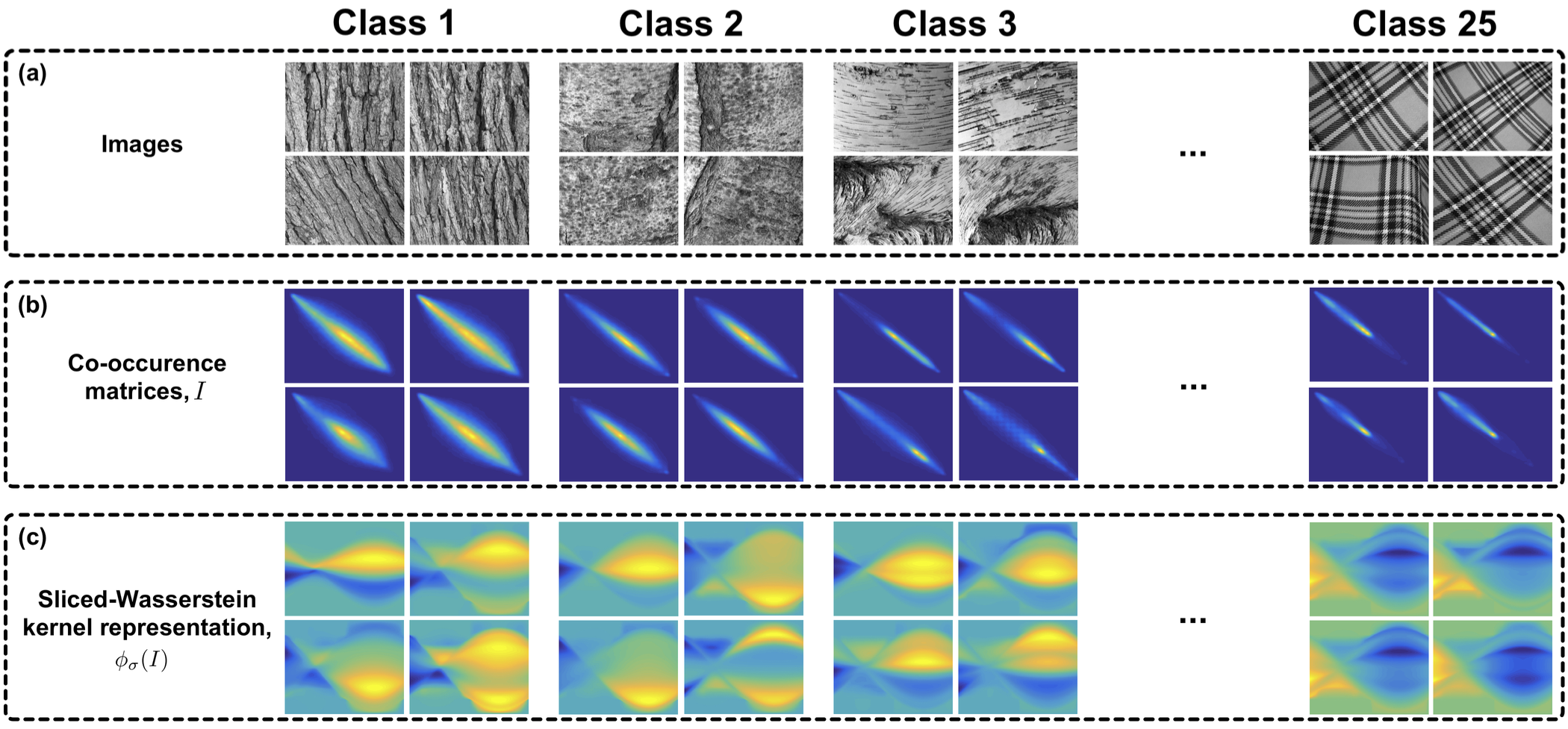}
\caption{The UIUC texture dataset with 25 classes (a), the corresponding calculated co-occurence matrices (b), and the kernel representation (i.e. $\phi_\sigma$) of the co-occurence matrices (c). }
\label{fig:uiuc}
\end{figure*}  

\begin{figure*}[t]
\centering
\includegraphics[width=\linewidth]{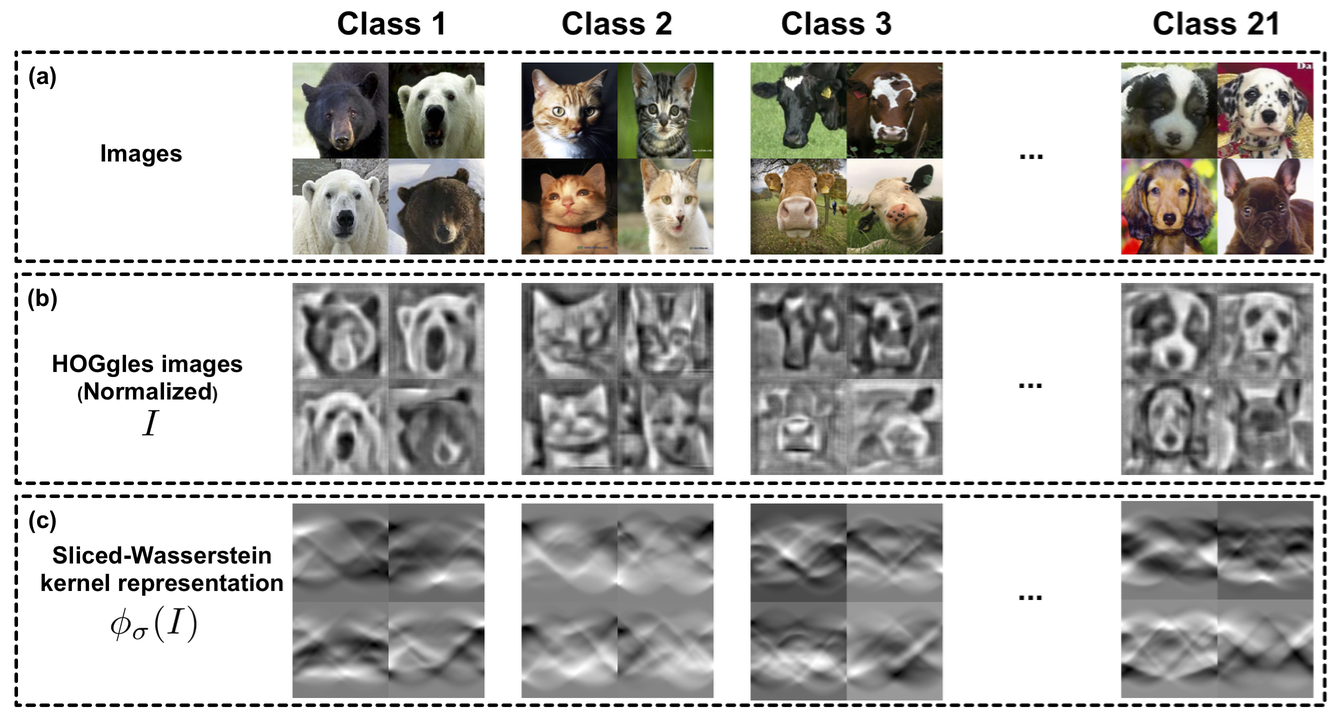}
\caption{The LHI animal face dataset with 21 classes (a), the corresponding calculated HOGgles representation (b), and the kernel representation (i.e. $\phi_\sigma$) of the HOGgles images (c). }
\label{fig:LHI}
\end{figure*}

%\begin{figure*}[t]
%\centering
%\includegraphics[width=\linewidth]{Figures/DataCVPR16.png}
%\caption{The texture and LHI animal face datasets, the corresponding calculated co-occurence and HOGgles representations, and the corresponding kernel representation (i.e. $\phi_\sigma$). The number of classes for the texture dataset is $N=25$ and for the LHI dataset is $N=21$.}
%\label{fig:LHI}
%\end{figure*}  

\section{Experimental Results}
\label{sec:results}

For our experiments we utilized  two image datasets, namely the University of Illinois Urbana Champaign (UIUC) texture dataset \cite{lazebnik2005sparse}  and the LHI animal face dataset \cite{si2012learning}. The texture dataset contains $25$ classes of texture images with $40$ images per class, which include a wide range of variations. The animal face dataset contains $21$ classes of animal faces with the average number of images per class being $114$. Figures \ref{fig:uiuc} (a) and \ref{fig:LHI} (a) show sample images from image classes for the texture and the LHI dataset, respectively. For the texture dataset we extract the gray-level co-occurence matrices for texture images and normalized them to obtain empirical joint probability density functions of co-occuring gray levels (See Figure \ref{fig:uiuc} (b)). On the other hand, for the animal face dataset, we used the normalized HOGgles images  \cite{vondrick2013hoggles} as a probability distribution representation of RGB animal face images (See Figure \ref{fig:LHI} (b)). The kernel representation $\phi_\sigma(I)$ for the extracted probability distributions is then calculated as shown in Figures \ref{fig:uiuc} (c) and \ref{fig:LHI} (c). We note that the fixed density $I_0$ for both datasets is chosen to be the average distribution over the entire dataset.  We also acknowledge that the HOGgles is not designed for feature extraction but rather for visualization of the extracted HoG features \cite{vondrick2013hoggles}, but our goal here is to show that for any extracted probability density features from images the Sliced Wasserstein kernels often outperform commonly used kernels.

First, the PCA of the input data, $I_1,...,I_N\in M$ (i.e. the co-occurence matrices for the texture images and the HOGgles images for the LHI dataset) as well as the kernel-PCA of the data with kernel $k(I_i,I_j)=\langle \phi_\sigma(I_i),\phi_\sigma(I_j)\rangle$ are calculated for both datasets. The reason behind choosing this kernel over the polynomial of degree $d>1$ or Gaussian Sliced Wasserstein kernel is simply that it is parameter free and the eigenvalue spectrum of the kernel matrix does not depend on hyperparameters. Figure \ref{fig:percentageEnergy} shows the cumulative percent variance (CPV) of the dataset captured by the principal components of $I_1,...,I_N \in M$ and $\phi_\sigma(I_1),...,\phi_\sigma(I_N)\in\V$  for both datasets. It can be seen that the variations in the datasets are captured more efficiently in the Sliced Wasserstein kernel space.

\begin{figure}
\centering
%\includegraphics[width=\columnwidth]{Figures/PercentageEnergy.png}\\
%(a)\\
%\includegraphics[width=\columnwidth]{Figures/PercentageEnergyLHI.png}\\
%(b)
\includegraphics[width=\columnwidth]{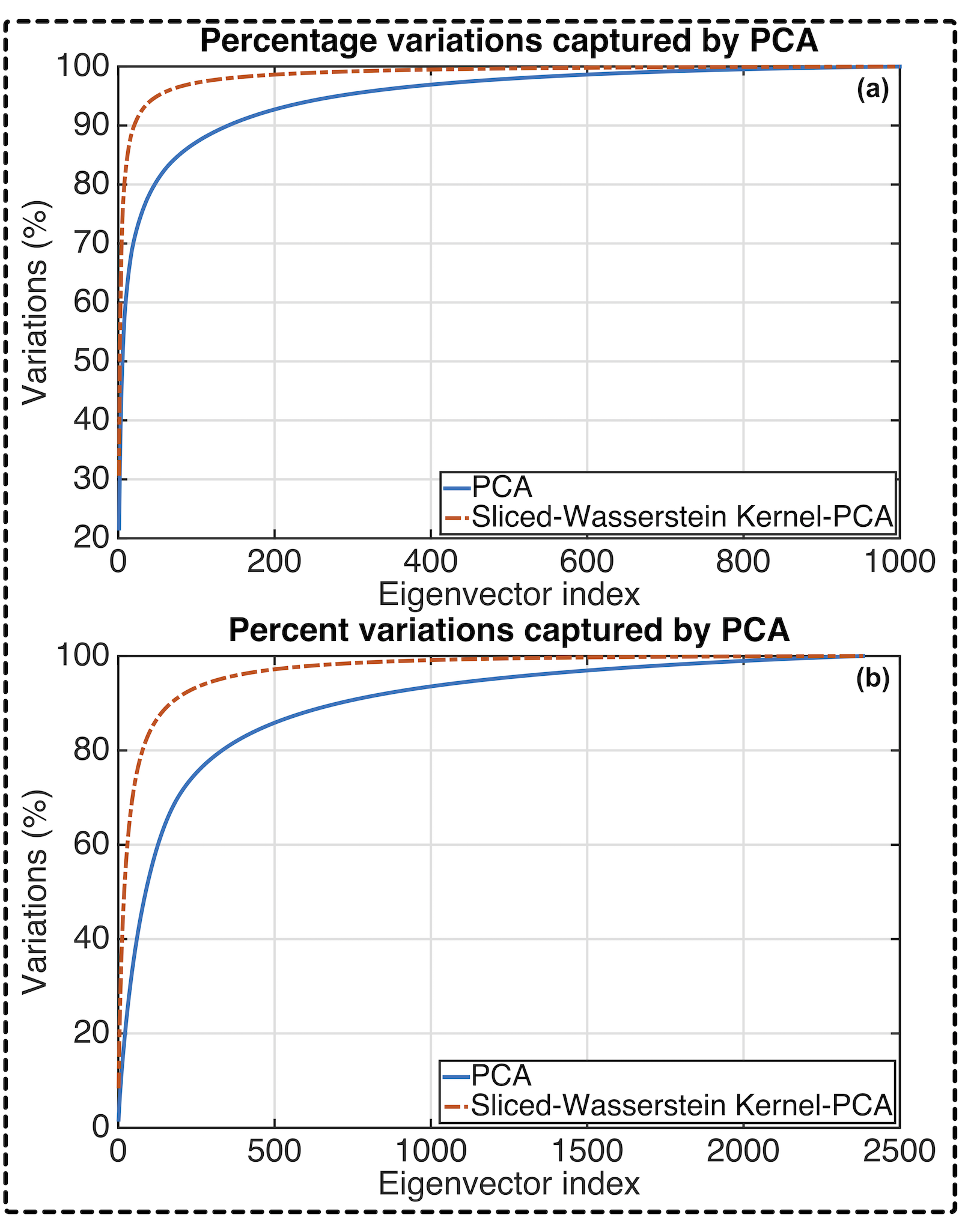}\\
\caption{Percentage variations captured by eigenvectors calculated from PCA and calculated from Sliced Wasserstein kernel PCA for the texture dataset (a) and for the animal face dataset (b).}
\label{fig:percentageEnergy}
\end{figure}

Next, we performed classification tasks on the texture and animal face datasets. Linear SVM, RBF kernel SVM, Sliced Wasserstein Gaussian kernel SVM, and the Sliced Wasserstein polynomial kernel were utilized for classification accuracy comparison.  A five fold cross validation scheme was used, in which $20\%$ of each class is held out for testing  and the rest is used for training and parameter estimation. The hyper parameters of the kernels are calculated through grid search. The classification experiments were repeated $100$ times and the means and standard deviations of the classification accuracies for both datasets are reported in Figure \ref{fig:accuraySVM}.

\begin{figure}
\centering
\includegraphics[width=\columnwidth]{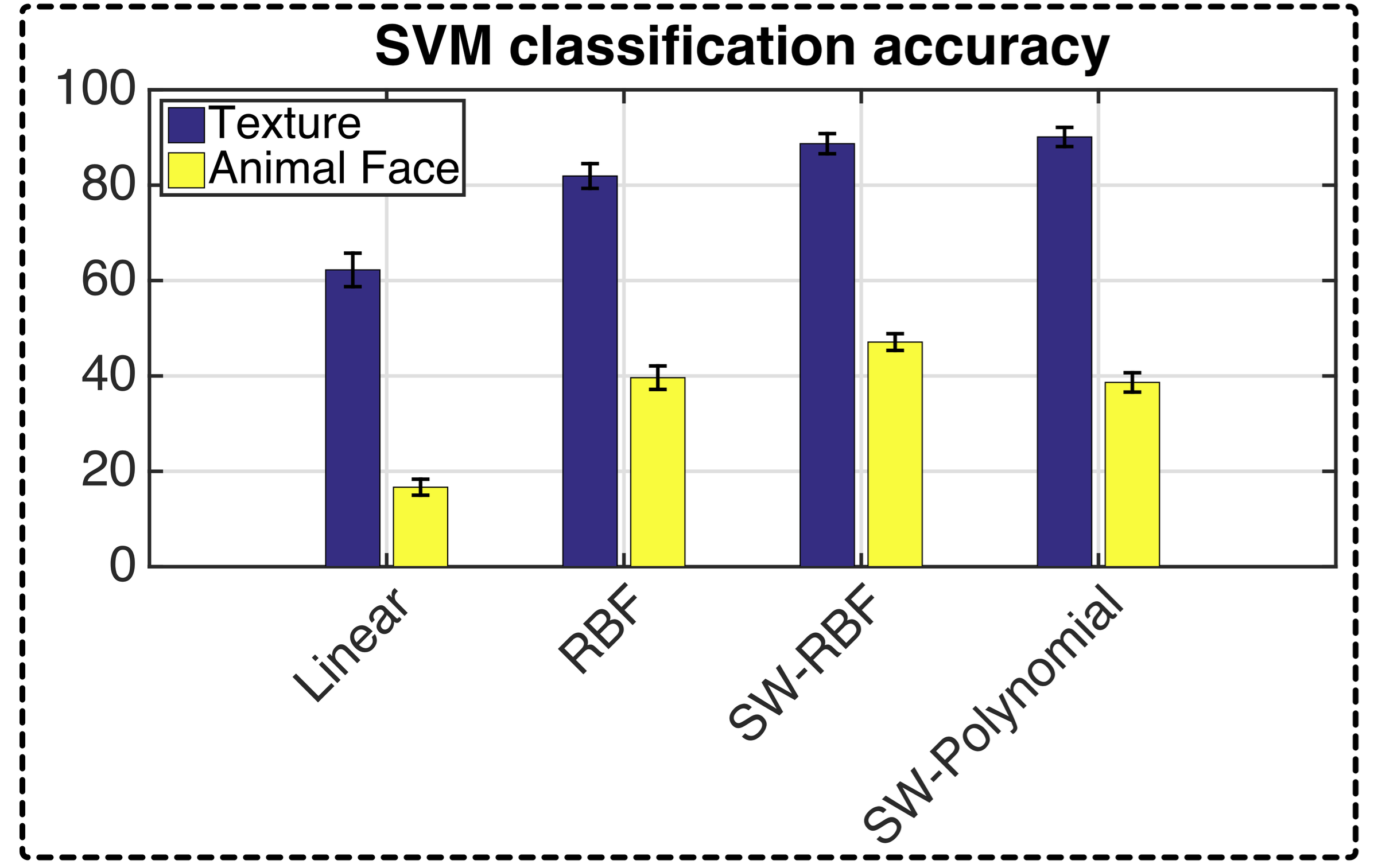}
\caption{Kernel SVM classification accuracy with linear kernel, radial basis function kernel (RBF), Sliced Wasserstein Gaussian Kernel (SW-RBF), and Sliced Wasserstein Polynomial Kernel (SW-Polynomial).}
\label{fig:accuraySVM}
\end{figure}

Finally, we perform clustering on the UIUC texture and the LHI animal face dataset. We utilized the k-means algorithm on the normalized co-occurence matrices and the HOGgles images, $I\in M$, and their corresponding representations in the kernel space, $\phi_\sigma(I)\in\V$ (i.e. kernel $k$-means). In order to be able to compare the within-cluster sum-of-squares (aka the inertia) we normalized the data by the average norm of $I_i$'s and $\phi_\sigma(I_i)$'s. We repeated the $k$-means experiment $100$ times and measured the within-cluster sum-of-squares and the V-measure \cite{rosenberg2007v}, which is a conditional entropy-based external cluster evaluation measure, at each iteration for both datasets.  Figure \ref{fig:clustereval}  shows the mean and standard deviation of the within-cluster sum-of-squares and the V-measure for $k$-means and Sliced Wasserstein kernel $k$-means. It can be seen that the Sliced Wasserstein Gaussian Kernel $k$-means provides better clusters which match the texture and animal face classes better, as it leads to higher V-measure values and lower inertia.  

\begin{figure}
\centering
\includegraphics[width=\columnwidth]{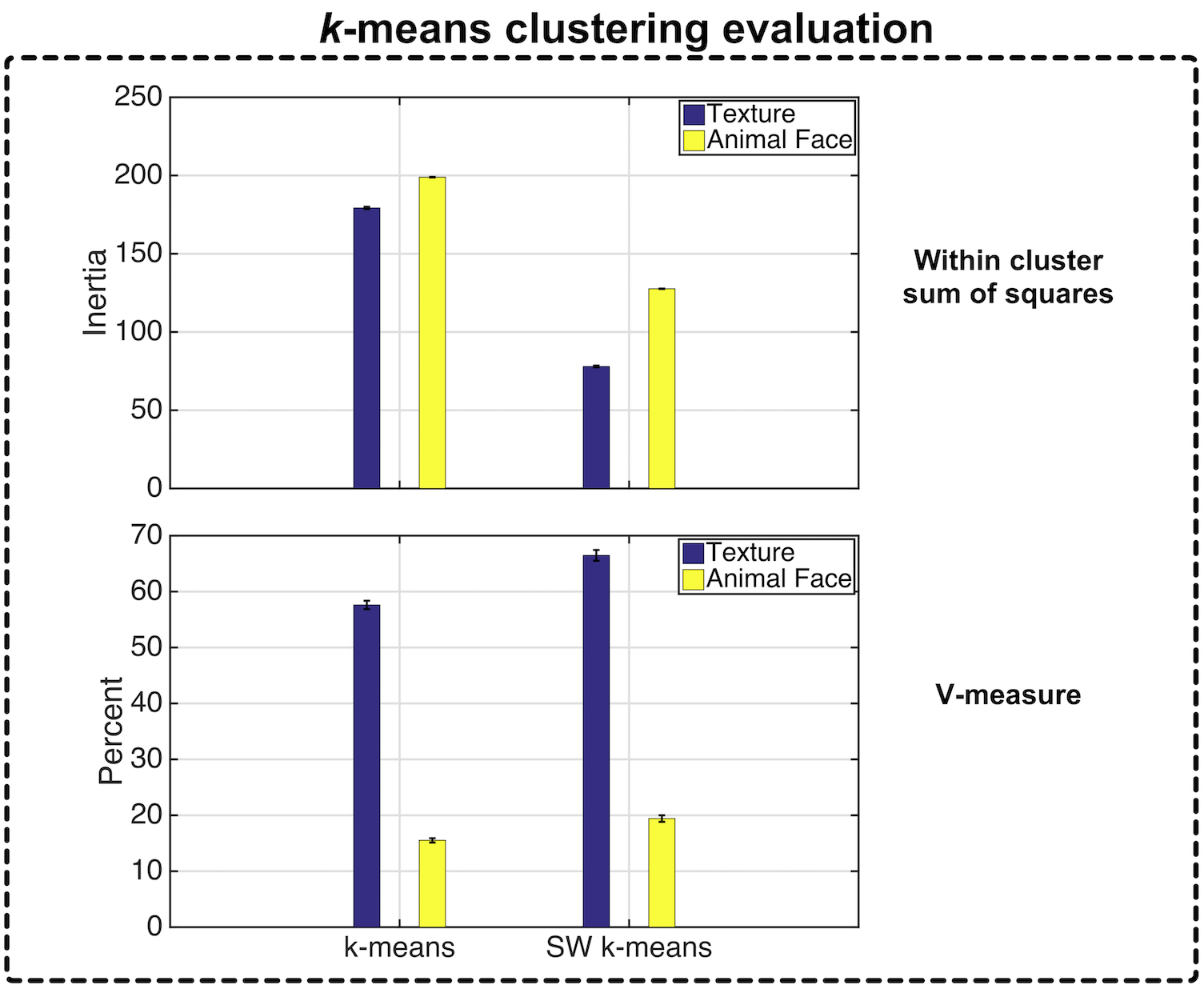}
\caption{Cluster evaluation for $k$-means and Sliced Wasserstein kernel $k$-means using the  within-cluster sum-of-squares measure (a), and the V-measure (b).}
\label{fig:clustereval}
\end{figure}

\section{Discussion}
\label{sec:conclusion}

In this paper, we have introduced a family of provably positive definite kernels for probability distributions based on the mathematics of the optimal transport and more precisely the Sliced Wasserstein distance. We denote our proposed family of kernels as the Sliced Wasserstein kernels. Following  the work of \cite{park2015cumulative,kolouri2015radon}, we provided an explicit nonlinear and invertible formulation for mapping probability distributions to the kernel space (aka feature space). Our experiments demonstrated the benefits of the Sliced Wasserstein kernels over the commonly used RBF and Polynomial kernels in a variety of pattern recognition tasks on probability densities. 

More specifically, we showed that utilizing a dimensionality reduction scheme like PCA with the Sliced Wasserstein kernel leads to capturing more of the variations of the datasets with fewer parameters. Similarly, clustering methods can benefit from the Sliced Wasserstein kernel as the clusters have higher values of V-measure and lower value of inertia. Finally, we demonstrated that the classification accuracy for a kernel classifier like the kernel SVM  can also benefit from the Sliced Wasserstein kernels.  

Finally, the experiments in this paper were focused on two-dimensional distributions. However, the proposed framework can be extended to higher-dimensional probability densities. We therefore intend to investigate the application of the Sliced Wasserstein kernel to higher-dimensional probability densities such as volumetric MRI/CT brain data. 

%{\small
\bibliographystyle{ieee}
\bibliography{Wasserstein_Kernel_arXiv}
%}

\end{document}